%% file: ms.tex
\documentclass[letterpaper, 10 pt, conference]{ieeeconf}  

\IEEEoverridecommandlockouts                              

\usepackage[english]{babel}

\makeatletter
\adddialect\l@ENGLISH\l@english
\makeatother

\usepackage{cite}

\usepackage{epsfig}
\usepackage{pst-all, graphics, graphicx, color}
\usepackage[crop=pdfcrop]{pstool}
\usepackage{subfigure}
\usepackage{amsmath, amssymb} 
\usepackage{epstopdf}
\usepackage{tikz}
\usepackage{multirow}
\usepackage{float}
\usepackage{algorithm, algorithmic}
\usepackage{dsfont}

\usepackage{amsthm}

\usepackage[makeroom]{cancel}

\floatstyle{ruled}
\newfloat{algorithm}{tbp}{loa}
\providecommand{\algorithmname}{Algorithm}
\floatname{algorithm}{\protect\algorithmname}

\makeatletter
\newcommand\footnoteref[1]{\protected@xdef\@thefnmark{\ref{#1}}\@footnotemark}
\makeatother

\renewcommand{\algorithmiccomment}[1]{\bgroup\hfill\scriptsize//~#1\egroup}

\input mathsym 

\newtheorem{theorem}{Theorem}

\begin{document}

\title{\LARGE \bf Merging Position and Orientation Motion Primitives}

\author{Matteo Saveriano$^{1}$, Felix Franzel$^{2}$, and Dongheui Lee$^{1,2}$%
\thanks{$^{1}$Institute of Robotics and Mechatronics, German Aerospace Center (DLR), We{\ss}ling, Germany {\tt matteo.saveriano@dlr.de}.}%
\thanks{$^{2}$Human-Centered Assistive Robotics, Technical University of Munich, Munich, Germany {\tt felix.franzel@tum.de, dhlee@tum.de}.}%
\thanks{This work has been supported by Helmholtz Association.}
}

\maketitle


\begin{abstract}
In this paper, we focus on generating complex robotic trajectories by merging sequential motion primitives. A robotic trajectory is a time series of positions and orientations ending at a desired target. Hence, we first discuss the generation of converging pose trajectories via dynamical systems, providing a rigorous stability analysis. Then, we present approaches to merge motion primitives which represent both the position and the orientation part of the motion. 
Developed approaches preserve the shape of each learned movement and allow for continuous transitions among succeeding motion primitives. Presented methodologies are theoretically described and experimentally evaluated, showing that it is possible to generate a smooth pose trajectory out of multiple motion primitives.     
\end{abstract}

\IEEEpeerreviewmaketitle

\input{sections/Introduction}

\input{sections/Merging_Approach}

\input{sections/Experiments}

\input{sections/Conclusions}

\input{sections/Appendix}

\bibliographystyle{IEEEtran}
\bibliography{mybib}
\end{document}

%% file: mathsym.tex
\def\tr{^{\top}}

\def\bfzero{\hbox{\bf 0}}

\def\bfa{{\mbox{\boldmath $a$}}}

\def\bfd{{\mbox{\boldmath $d$}}}
\def\bfe{{\mbox{\boldmath $e$}}}
\def\bff{{\mbox{\boldmath $f$}}}
\def\bfg{{\mbox{\boldmath $g$}}}

\def\bfp{{\mbox{\boldmath $p$}}}
\def\bfq{{\mbox{\boldmath $q$}}}
\def\bfr{{\mbox{\boldmath $r$}}}

\def\bfv{{\mbox{\boldmath $v$}}}
\def\bfw{{\mbox{\boldmath $w$}}}
\def\bfx{{\mbox{\boldmath $x$}}}

\def\bfD{{\mbox{\boldmath $D$}}}

\def\bfI{{\mbox{\boldmath $I$}}}

\def\bfK{{\mbox{\boldmath $K$}}}

\def\bfS{{\mbox{\boldmath $S$}}}

\def\bfepsilon{{\mbox{\boldmath $\epsilon$}}}

\def\bfomega{{\mbox{\boldmath $\omega$}}}

%% file: sections/Introduction.tex
\section{Introduction}
Robots operating in everyday environments will execute a multitude of tasks ranging from simple motions to complex activities consisting of several actions performed on different objects. Hand programming of all these tasks is not feasible. Hence, researchers have investigated how to acquire novel tasks in an intuitive manner \cite{Schaal_99, CalinonLee19}. A possible solution is to demonstrate the task to execute, for example by physically guiding the robot towards the task completion  \cite{Lee_11, Saveriano_15}. Collected data are then used for motion planning.

Motion planning with dynamical systems has gained attention in the robot learning community and researchers have developed several approaches to represent demonstrations as dynamical systems \cite{DMP, TP-DMP, Blocher17,  Clf, Neumann2015, Perrin16, Saveriano18, Kronander15, Gribovskaya09, Pastor11, Ude14, Zeestraten17}. Dynamical systems are used to plan in joint or Cartesian space, and, in Cartesian space, to encode both position and orientation \cite{Gribovskaya09, Pastor11, Ude14, Zeestraten17}. Moreover, robots driven by stable systems are able to reproduce complex paths \cite{Blocher17, Clf, Neumann2015, Perrin16}, to incrementally update a predefined skill \cite{Saveriano18, Kronander15}, and to avoid possible collisions \cite{DS_avoidance, Saveriano13, Saveriano14, Saveriano17, Hoffmann09}. 

Complex robotic tasks, consisting of several actions, can be obtained by sequencing multiple motion primitives \cite{Kulic12, Muehlig12, Manschitz_15, Caccavale17, Caccavale18}. As in \cite{Manschitz_15, Caccavale17, Caccavale18}, this work represents the motion primitives as Dynamic Movement Primitives (DMP) \cite{DMP}, but other choices are possible \cite{Kulic12, Muehlig12}. 
Given a set of DMPs, the problem arises of how the DMPs can be merged to generate a unique and smooth trajectory without stopping at the end of each motion primitive. Pastor et al. \cite{Pastor_09} address this problem by activating the succeeding motion primitive when the velocity of the current primitive is smaller than a threshold. The succeeding primitive is initialized with the state reached by the previous one at the switching point. This avoids jumps in the velocity but it may cause jumps in the acceleration. To avoid jumps in acceleration, \cite{Nemec_09} augments the DMP with a low-pass filter. {Nevertheless, when positions and velocities of the consecutive primitives at the  switching  point are significantly  different, the trajectory has to be filtered a lot introducing a delay with consequent large deviations from the demonstration. The approach in \cite{Kober_10} has been proposed to learn hitting motions in table tennis, but it can be used to merge motion primitives. In \cite{Kober_10}, the DMP is augmented with a  moving  target  and  final velocity. Hence, each DMP reaches a certain position with a given velocity (different from zero) which are used to initialize the succeeding DMP. 
Instead of switching the DMPs, the approach in \cite{Kulvicius_11} creates a unique DMP by overlapping sequential movements. The unique DMP preserves the shape of each overlapped motion.} 

\begin{figure}[t]
	\centering
	\includegraphics[width=0.9\columnwidth]{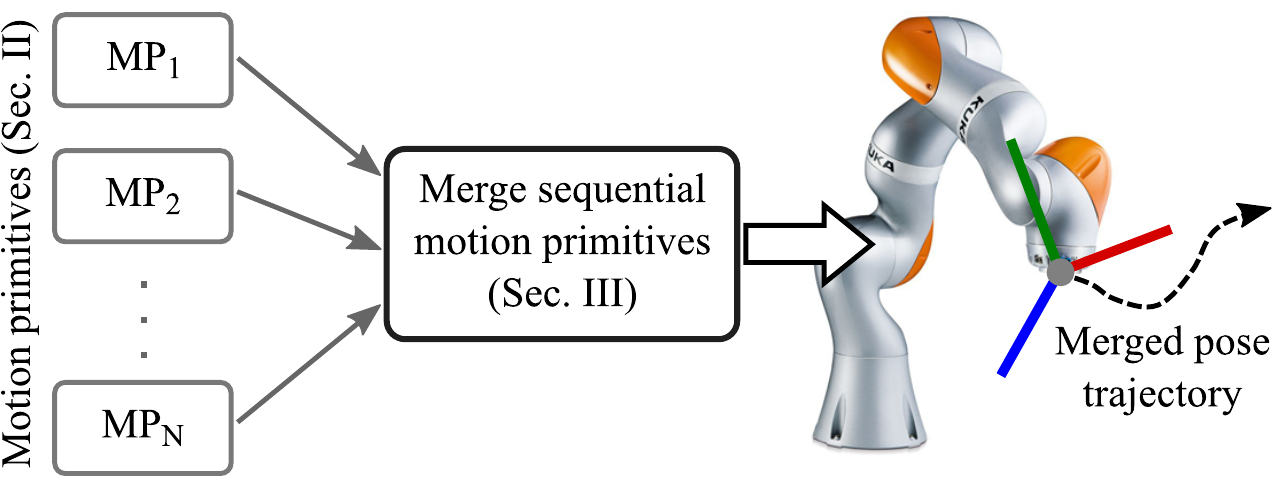}
	 \vspace{-0.2cm}
    \caption{Motion primitives are merged to generate a smooth robot trajectory.}
    \label{fig:overview}
     \vspace{-0.5cm}
\end{figure}

Aforementioned approaches are effective when the dynamical system is used to represent Cartesian positions or joint angles. However, they do not consider the orientation part of the motion. DMP formulations capable of encoding Cartesian orientation have been proposed in \cite{Pastor11,Ude14}, but without considering the problem of merging multiple movements. 
In this work, we first describe how DMP and unit quaternions are used to encode orientation trajectories and present a rigorous stability analysis that is missing in the related literature {\cite{Pastor11, Ude14}}. We then extend the approaches in \cite{Pastor_09}, \cite{Kober_10}, and \cite{Kulvicius_11} to merge sequential DMP representing both position and orientation (see Fig. \ref{fig:overview}). We use unit quaternions to represent the orientation and rely on quaternion algebra to define all the mathematical operations needed to merge the learned motions. Finally, we compare the presented approaches on simulated and real data in order to underline advantages and drawbacks of each approach.

%% file: sections/Merging_Approach.tex
\section{Cartesian Pose Motion Primitives}\label{sec:pose_dmp}
In this section, we describe how Cartesian poses are represented in the dynamic movement primitives (DMP) framework \cite{DMP} and provide a stability analysis. 
\subsection{Position DMP}\label{subsec:position_dmp}
 Following the representation introduced by Park et al.~\cite{Park08}, Cartesian positions are generated via the second-order dynamical system (time dependency is omitted for simplicity) 
\small
\begin{subequations}
\begin{align}
\tau \dot{\bfp} &= \bfv, \label{eq:dmp_lin_vel}\\
\tau \dot{\bfv} &= \bfK^p\left[(\bfp_d - \bfp) - \bfd^{p}_0(h) + \bff^p(h)\right] -\bfD^p\bfv,
\label{eq:dmp_lin_acc}
\end{align}
\end{subequations}
\normalsize
where $\bfp \in \mathbb{R}^3$ is the position, $\dot{\bfp} = \bfv \in \mathbb{R}^3$ is the linear velocity, and $\dot{\bfv} \in \mathbb{R}^3$ is the linear acceleration. The time scaling factor $\tau$ can be adapted to change the duration of the movement without changing the path. The positive definite matrices $\bfK^p$, $\bfD^p \in \mathbb{R}^{3\times 3}$ are  linear stiffness and damping gains respectively. The scalar $h$ is an exponentially decaying clock signal, obtained by integrating the so-called canonical system $\tau\dot{h} = -\gamma h$, %
with $\gamma > 0$. The clock signal is $h=1$ at the beginning of the motion and it exponentially converges to zero. The term $\bfd^{p}_0(h) = (\bfp_d - \bfp_{0})h$ in \eqref{eq:dmp_lin_acc} prevents a jump at the beginning of the motion and it vanishes for $h\rightarrow 0$. The forcing term $\bff^{p}(h)$ in \eqref{eq:dmp_lin_acc} is defined as
\small
\begin{equation}
\bff^{p}(h) = \frac{\sum_{i=1}^{N}\bfw_{i}\psi_{i}(h)}{\sum_{i=1}^{N}\psi_{i}(h)}h,\quad \psi_{i}(h) = e^{-a(h-c_{i})^2}. \label{eq:force_term}\\
\end{equation}
\normalsize
Given the amplitude $a$ and the centers $c_{i}$, the parameters $\bfw_{i}$ are learned from demonstration using weighted least square \cite{DMP}. From \eqref{eq:force_term}, it is clear that $\bff^{p}(h)$ vanishes for $h\rightarrow 0$.

\subsection{Orientation DMP}\label{subsec:orientation_dmp}
Dynamic movement primitives, commonly used to represent Cartesian or joint position, have been extended to represent Cartesian orientation \cite{Pastor11, Ude14}. The approaches in \cite{Pastor11} and \cite{Ude14} use a different definition of the orientation error, as detailed later in this section. In this work, orientation is represented by a unit quaternion $\bfq = [\eta, \bfepsilon\tr]\tr \in \mathcal{S}^3$, where $\mathcal{S}^3$ is the unit sphere in the $3$D space. Unit quaternions have less parameters compared to rotation matrices ($4$ instead of $9$). Compared to other representations, like Euler angles,  they are uniquely defined and have no singularities if rotations are restricted to one hemisphere of $\mathcal{S}^3$ \cite{Siciliano_book}. A DMP for the orientation is defined as
\small
\begin{subequations}
\begin{align}
\tau \dot{\bfq} &= \frac{1}{2}\tilde{\bfomega} \ast \bfq,  \label{eq:dmp_ang_vel}\\
\tau \dot{\bfomega} &= \bfK^q\left[\bfe_o(\bfq_d , {\bfq}) - \bfd^{q}_0(h) + \bff^q(h)\right] -\bfD^q\bfomega, 
\label{eq:dmp_ang_acc}
\end{align}
\end{subequations}
\normalsize
where\footnote{{The DMP formulation in \eqref{eq:dmp_ang_vel}--\eqref{eq:dmp_ang_acc} is also adopted in \cite{Pastor11} but using $\bfe_o(\bfq , {\bfq}_d) = -\bfe_o(\bfq_d , {\bfq})$ as orientation error.}} $\bfq \in \mathcal{S}^3$ is the unit quaternion, $\bfomega \in \mathbb{R}^3$ and  $\dot{\bfomega} \in \mathbb{R}^3$ are the angular velocity and acceleration respectively, and $\tau$ is a temporal scaling factor. The symbol $\ast$ indicates the product of two quaternions defined in \eqref{eq:quaternion_product}, $\bfe_o(\cdot,\cdot) \in \mathbb{R}^3$ is the error between two quaternions, and the quantity $\tilde{\bfomega}$ is the angular velocity quaternion, i.e. $\tilde{\bfomega} = [0, \bfomega\tr]\tr$. In other words, $\tilde{\bfomega}$ is a quaternion with zero as scalar part and the angular velocity as vector part. The positive definite matrices $\bfK^q$, $\bfD^q \in \mathbb{R}^{3\times 3}$ are angular stiffness and damping gains respectively. The clock signal is the same used for the position ($\tau\dot{h} = -\gamma h$). The term $\bfd^{q}_0(h) = \bfe_o(\bfq_d \ast {\bfq}_0)h$ in \eqref{eq:dmp_ang_acc} prevents a jump at the beginning of the motion and it vanishes for $h\rightarrow 0$. The nonlinear forcing term $\bff^{q}(h)$ in \eqref{eq:dmp_ang_acc} is defined as in \eqref{eq:force_term}, it is learned from demonstration, and it vanishes for $h\rightarrow 0$. The quaternion rate \eqref{eq:dmp_ang_vel} is integrated by means of \eqref{eq:quaterion_rate_integration}. 

There are two key differences between position DMP in \eqref{eq:dmp_lin_vel}--\eqref{eq:dmp_lin_acc} and orientation DMP in \eqref{eq:dmp_ang_vel}--\eqref{eq:dmp_ang_acc}. First, the relationship between the time derivative of the quaternion $\dot{\bfq}$ and the angular velocity in \eqref{eq:dmp_ang_vel} is nonlinear, while the derivative of the position equals the linear velocity in \eqref{eq:dmp_lin_vel}. Second, the error between two quaternions $\bfe_o(\cdot,\cdot)$ in \eqref{eq:dmp_ang_acc} is a nonlinear function and it has multiple definitions, while the error between two positions in \eqref{eq:dmp_lin_acc} is simply their difference. In robotics and control, the orientation error between quaternions $\bfq_1$ and $\bfq_2$ is typically defined as $\bfe_o = \text{vec}(\bfq_1 \ast \overline{\bfq}_2)$ \cite{Siciliano_book, Yuan88}, where the function $\text{vec}(\bfq)$ returns the vector part of $\bfq$. This definition of the orientation error is used in \cite{Pastor11} for orientation DMP, while Ude et al. \cite{Ude14} propose to use the quantity $2\log(\bfq_1 \ast \overline{\bfq}_2)$ as orientation error, where the \textit{logarithmic map} $\log(\cdot)$ is defined as in \eqref{eq:logarithmic_map}. 

\subsection{Stability analysis}\label{subsec:stability}
The stability of the position DMP in \eqref{eq:dmp_lin_vel}--\eqref{eq:dmp_lin_acc} is trivially proved. Indeed, $\bfd^{p}_0(h)$ and $\bff^{p}(h)$ vanish for the time $t \rightarrow +\infty$ and \eqref{eq:dmp_lin_vel}--\eqref{eq:dmp_lin_acc} become a linear system. Hence, the positive definiteness of $\bfK^p$ and $\bfD^p$ is sufficient to conclude that the dynamical system \eqref{eq:dmp_lin_vel}--\eqref{eq:dmp_lin_acc} asymptotically converges to $\bfp_d$ with zero velocity.
For orientation DMPs, instead, it is interesting to prove the following stability theorem:
\begin{theorem}
\label{th:ori_dmp_stab}
The orientation DMP in \eqref{eq:dmp_ang_vel}--\eqref{eq:dmp_ang_acc}, with clock signal $h$ defined such that $h \rightarrow 0 $ for $t \rightarrow +\infty$ and orientation error defined as $\bfe_o(\bfq_1, \bfq_2) = \text{\em{vec}}(\bfq_1 \ast \overline{\bfq}_2)$, globally asymptotically converges to $\hat{\bfq} = \bfq_d$ with $\hat{\bfomega} = \bfzero$. 
\end{theorem}  
\begin{proof}
Recall that the non-linearities in \eqref{eq:dmp_ang_vel}--\eqref{eq:dmp_ang_acc} are smooth functions and that the time dependancy introduced by $h$ vanishes for $t\rightarrow +\infty$. Hence, \eqref{eq:dmp_ang_vel}--\eqref{eq:dmp_ang_acc} are an \textit{asymptotically autonomous differential system} and the stability can be proved analyzing its asymptotic behavior \cite{Markus56}. In other words, we have to prove the stability of
\small
\begin{equation}
\dot{\bfq} = \frac{1}{2}\tilde{\bfomega} \ast \bfq, \quad \dot{\bfomega} = \bfK^q\text{vec}(\bfq_d \ast \overline{\bfq}) -\bfD^q\bfomega, \label{eq:dmp_as}
\end{equation}
\normalsize
where we set $\tau = 1$ without loss of generality. The stability of the non-linear system \eqref{eq:dmp_as} is proved with the Lyapunov method \cite{Slotine91}, using the Lyapunov candidate
\small
\begin{equation}
V(\bfx) = \left(\eta_d - \eta\right)^2 + \Vert\bfepsilon_d-\bfepsilon\Vert^2 + \frac{1}{2}\bfomega\tr(\bfK^q)^{-1}\bfomega,\label{eq:lyapunov_fun_rel_ori}
\end{equation}
\normalsize
where the state $\bfx = [\bfq\tr,\,\bfomega\tr]\tr$, $\bfq= [\eta,\,\bfepsilon\tr]\tr$, and $\bfq_d = [\eta_d,\,\bfepsilon_{d}\tr]\tr$. The candidate Lyapunov function in \eqref{eq:lyapunov_fun_rel_ori} is positive definite and it vanishes only at the equilibrium $\hat{\bfx} = [{\bfq_d}\tr,\,\bfzero\tr]\tr$. The time derivative of $V(\bfx)$ is
\small
\begin{equation*}
\begin{split}
\dot{V} &= -2\left(\eta_d - \eta\right)\dot{\eta} -2\left(\bfepsilon_d-\bfepsilon\right)\dot{\bfepsilon}+\bfomega\tr(\bfK^q)^{-1}\dot{\bfomega}\\
&= \left(\eta_d - \eta\right)\bfepsilon\tr\bfomega + \left(\bfepsilon_d-\bfepsilon\right)\tr\left(\eta\bfI - \bfS(\bfepsilon)\right)\bfomega+\bfomega\tr(\bfK^q)^{-1}\dot{\bfomega}
\end{split}
\end{equation*}
\normalsize
where we used $\dot{q}$ in \eqref{eq:dmp_as} and the quaternion propagation \eqref{eq:quaternion_propagation}. Considering the definition of $\dot{\bfomega}$ in \eqref{eq:dmp_as}, we obtain that
\small
\begin{equation*}
\begin{split}
\dot{V}&= \left(\eta_d - \eta\right)\bfepsilon\tr\bfomega + \left(\bfepsilon_d-\bfepsilon\right)\left(\eta\bfI - \bfS(\bfepsilon)\right)\bfomega\\
&\quad+\bfomega\tr\cancel{(\bfK^q)^{-1}\bfK^q}\text{vec}(\bfq_d \ast \overline{\bfq}) -\bfomega\tr(\bfK^q)^{-1}\bfD^q\bfomega
\end{split}
\end{equation*}
\normalsize
Considering the quaternion product in \eqref{eq:quaternion_product} and that $\bfS(\bfa)\bfa = \bfzero$ if $\bfS(\cdot)$ is a \textit{skew-symmetric} matrix, we obtain that
\small
\begin{equation*}
\dot{V} = -\bfomega\tr(\bfK^q)^{-1}\bfD^q\bfomega + \cancel{\bfomega\tr\left(\text{vec}(\bfq_d \ast \overline{\bfq}) - \text{vec}(\bfq_d \ast \overline{\bfq})\right)}
\end{equation*}
\normalsize
The matrix $(\bfK^q)^{-1}\bfD^q$, where $(\bfK^q)^{-1}$ and $\bfD^q$ are positive definite matrices, is positive definite iff $(\bfK^q)^{-1}\bfD^q = \bfD^q(\bfK^q)^{-1}$. This property can be guaranteed, for example, by assuming that $\bfK^q$ and $\bfD^q$ are diagonal matrices. If $(\bfK^q)^{-1}\bfD^q$ is a positive definite matrix, then $\dot{V} \leq 0$ and $\dot{V}$ vanishes iff $\bfomega = \bfzero$. The LaSalle's invariance theorem \cite{Slotine91} allows to conclude the stability of \eqref{eq:dmp_as}. 
\end{proof}

In \cite{Ude14}, authors use $\bfe_o = 2\text{log}(\bfg \ast \overline{\bfq})$ . With this choice, the stability can be shown using $V(\bfx) = \left(\eta_d - \eta\right)^2 + \Vert\bfepsilon_d-\bfepsilon\Vert^2 + 0.5\bfomega\tr\bfomega$ as Lyapunov function and selecting $\bfK^q = \frac{\Vert \text{vec}(\bfq_d \ast \overline{\bfq})\Vert}{2\arccos(\text{scal}(\bfq_d \ast \overline{\bfq}))}\bfI$ as stiffness gain. This non-linear stiffness gain has a singularity when $\bfq_d$ and $\bfq$ are aligned. In this work, we use the vector-based quaternion error $\bfe_o = \text{vec}(\cdot,\cdot)$ to avoid  non-linearity and singularity in the gain matrices.

\section{Merging Pose Motion Primitives}\label{sec:merge_dmp}
Motion primitives can be combined to execute complex robotics tasks \cite{Caccavale17, Manschitz_15, Caccavale18}. In this section, we present three different approaches to merge pose DMPs. Each of them follows a different idea on how to smoothly transition between successive DMPs. We assume that $L$ sequential pose DMPs are given. Each DMP converges to a certain position $\bfp_d^l$ and orientation $\bfq_d^l$ for $l=1,\ldots,L$. In all the presented approaches the clock signal vanishes for $t\rightarrow+\infty$. As discussed in Sec. \ref{subsec:stability}, this is sufficient to guarantee the convergence to the last goal $\bfp_d^L$, $\bfq_d^L$. Note that blue text is used in the equations to highlight the differences between the approaches in this section and the pose DMP in Sec. \ref{sec:pose_dmp}. 


\subsection{First Approach}\label{subsec:first_approach}
The method described in~\cite{Pastor_09} originates from the assumption that any DMP reaches the end position with zero velocity and zero acceleration. This means that once a motion is fully executed it will come to a full stop 
and that, close to the goal position, the robot moves with a decreasing velocity. In order to combine $L$ motion primitives, one can stop the current motion when the norm of the velocity is smaller than a certain threshold and start the next primitive. The next primitive is initialized with the state of the previous one ($\bfp_{ne}=\bfp_{pr}$, $\bfv_{ne}=\bfv_{pr}$) to avoid discontinuities. This applies to orientation by initializing the state of the next DMP as $\bfq_{ne}=\bfq_{pr}$, $\bfomega_{ne}=\bfomega_{pr}$. Note that the approach applies to any second-order dynamical system including DMPs. 


\subsection{Second Approach}\label{subsec:second_approach}
\begin{figure}[t]
	\centering
	{%
    \subfigure[Position]{\centering {\includegraphics[width=0.49\columnwidth]{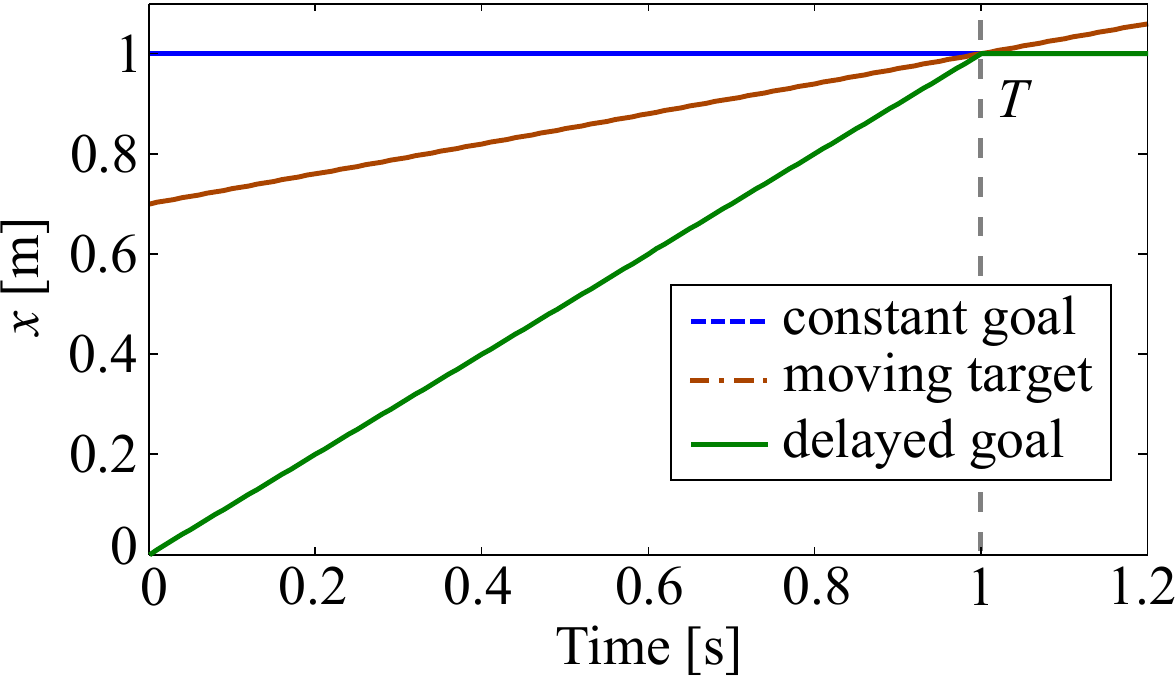}}}
    \subfigure[Quaternion]{\centering {\includegraphics[width=0.49\columnwidth]{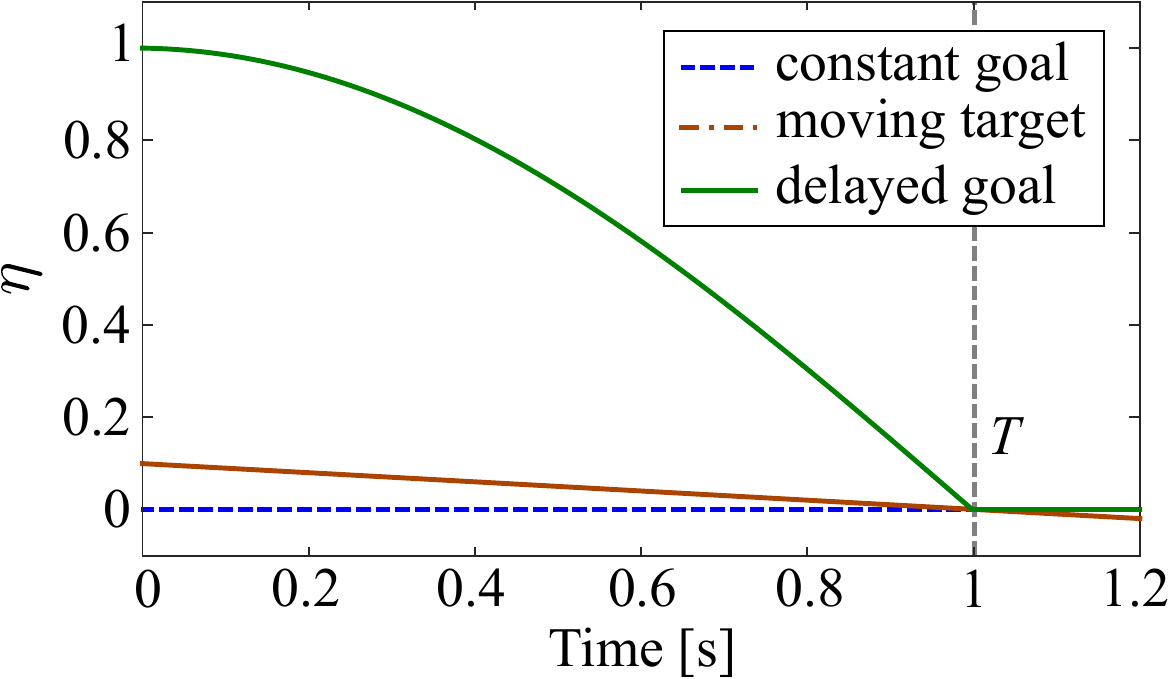}}}
    }%
    \vspace{-0.5cm}
    	\caption{The constant goal, moving target, and delayed goal obtained obtained with $\bfp(0) = [0,0,0]\tr\,$m, $\bfp_d = [1,0,0]\tr\,$m, $\bfq(0) = [1,0,0,0]\tr$, $\bfq_d = [0,1,0,0]\tr$, $\bfv_d = [0.3,0.3,0.3]\tr\,$m/s, $\bfv_d = [0.3,0.3,0.3]\tr\,$m/s, $\bfomega_d = [0.2,0.2,0.2]\tr\,$m/s, $\delta t = 0.01\,$s, and $T=1\,$s. Only $x$ for the position and $\eta$ for the quaternion are shown for a better visualization.}
    	\label{fig:clock_and_goal}
    	 \vspace{-0.5cm}
\end{figure} 
The approach in \cite{Kober_10} allows to cross the goal position of a DMP with a non-zero velocity. This is achieved by allowing the DMP to track a position target that moves at a given velocity. Hence, the linear acceleration in \eqref{eq:dmp_lin_acc} becomes
\small
\begin{equation*}
\tau \dot{\bfv} = \bfK^p\left[(\textcolor{blue}{\bfp_m^l} - \bfp)\textcolor{blue}{(1-h)} + \bff^p(h)\right]+\bfD^p(\textcolor{blue}{\bfv_d^l} - \bfv)\textcolor{blue}{(1-h)},
\end{equation*}
\normalsize
where $\bfv_d^l$ is the chosen final linear velocity of the $l$-th DMP and the moving target $\bfp_m^l$ is defined as
\small
$\bfp_{m}^l(t) = \bfp_{m}^l(0) -{\bfv_d^l}\frac{ \tau \ln(h) }{ \gamma }, \quad  \bfp_{m}^l(0) = \bfp_d^l-T^l\bfv_d^l, $
\normalsize
where~$\bfp_d^l$ is the goal position and~$T^l$ is the time duration of the $l$-th DMP. The moving target~$\bfp_{m}^l$ is designed to reach the goal position at~$\bfp_{m}^l(T^l) = \bfp_d^l$ (see Fig. \ref{fig:clock_and_goal}(a)). This is because the term $- \tau\ln(h) / \gamma$  represents the time if $h$ is defined by the canonical system $\tau\dot{h} = -\gamma h$. The initial position of the moving target $ \bfp_{m}^l(0)$ is computed by moving the goal position $\bfp_d^l$ for $T^l$ at constant velocity $-\bfv_d^l$. High accelerations at the beginning of the movement are avoided by the prefactor~$(1-h)$, that replaces the term $\bfd^{p}_0(h)$ in \eqref{eq:dmp_lin_acc}. 

The presented idea is here extended to unit quaternions. The angular acceleration in \eqref{eq:dmp_ang_acc} is rewritten as
\small
\begin{equation*}
\tau \dot{\bfomega} = \bfK^q\left[\bfe_o(\textcolor{blue}{\bfq_m^l} ,\bfq)\textcolor{blue}{(1-h)} + \bff^q(h)\right]+\bfD^q(\textcolor{blue}{\bfomega_d^l} - \bfomega)\textcolor{blue}{(1-h)},
\end{equation*}
\normalsize
where $\bfomega_d^l$ is the chosen final angular velocity of the $l$-th DMP and $\bfe_o(\bfq_m^l,\bfq) = \text{vec}(\bfq_m^l \ast \overline{\bfq})$ as detailed in Sec. \ref{subsec:orientation_dmp}. High angular accelerations at the beginning of the motions are prevented by the prefactor~$(1-h)$ that replaces the term $\bfd^{q}_0(h)$ used in \eqref{eq:dmp_ang_acc}. The moving target $\bfq_m^l$ is defined as
\small
\begin{equation}
\begin{aligned}
 &\bfq_m^l(t) = \exp\left(-\frac{\tau\ln(h)}{2\gamma} \bfomega_d^l \right) \ast \bfq_m^l(0), \\
 &\bfq_m^l(0) = \exp\left(-\frac{T^l}{2} \bfomega_d^l \right) \ast \bfq_d^l,
 \end{aligned}
\end{equation}
\normalsize
where~$\bfq_d^l$ is the goal orientation and~$T^l$ is the time duration of the $l$-th DMP. The initial orientation of the moving target $ \bfq_{m}^l(0)$ is computed by moving the goal orientation $\bfq_d^l$ for $T^l$ at constant velocity $-\bfomega_d^l$. Considering the definitions of the exponential map $\exp(\cdot)$ and the quaternion product $\ast$ in \eqref{eq:exponential_map} and \eqref{eq:quaternion_product} respectively, it is straightforward to verify that $\bfq_{m}^l$ reaches the goal quaternion at~$\bfq_{m}^l(T^l) = \bfq_d^l$ (Fig. \ref{fig:clock_and_goal}(b)).  



Having now the ability to cross each goal after $T^l$ with a non-zero velocity, we can combine multiple motion primitives. Given two consecutive DMPs $l$ and $l+1$, we run $l$ for $T^l$ seconds and then switch to $l+1$. To avoid discontinuities, we initialize the state of $l+1$ with the final state of $l$ \cite{Pastor_09}. 
\subsection{Third Approach}\label{subsec:thirdapproach}
The approach in \cite{Kulvicius_11} merges multiple DMPs into a single, more complex one. 
In \cite{Kulvicius_11}, the canonical system is 
\small
\begin{equation} 
\dot{h} = - \alpha_{ h }e^{\frac{ \alpha_{h} }{ \delta_{ t }}(\tau T-t)}
 / [ 1 + e^{\frac{ \alpha_{h} }{ \delta_{ t }}(\tau T-t)}]^{ 2 },
 \label{eq:sigmoidaldecay}
\end{equation}
\normalsize
where~$ \alpha_{ h }$ defines the steepness of the sigmoidal decay function $h$ centred at the time moment $T$. The value of~$h$ is~$h = 1$ for $t < T -\delta$, where $\delta$ depends on the steepness $\alpha_{ h }$, and then it decays to~$h=0$. The linear acceleration in \eqref{eq:dmp_lin_acc} becomes
\small
\begin{equation}
\tau \dot{\bfv} = \bfK^p(\textcolor{blue}{\bfp_m^l} - \bfp) + \bfK^p\bff^p(h) - \bfD^p\bfv,
\label{eq:dmp_lin_acc_3}
\end{equation}
\normalsize
while the linear velocity in \eqref{eq:dmp_lin_vel} is the same. {The moving target $\bfp_m^l$, called delayed goal function in \cite{Kulvicius_11}, is defined as
\small
\begin{equation} 
\tau \dot{\bfp}_m^l= 
\begin{cases}
\frac{ \delta{ t } }{ T^{l } }(\bfp_d^l-\bfp^l(0)), & \sum\limits_{ k=1 }^{ l-1 }{ T^{ k } \leq t \leq \sum\limits_{ k=1 }^{ l }{ T^{ k } }} \\
[0,\,0,\,0]\tr, &\text{otherwise}
\end{cases} ,
\label{eq:lineardelayedgoalfunction}
\end{equation}
\normalsize
where~$\bfp^l(0)$ and~$\bfp_d^l$ are the initial and goal position of the $l$-th DMP, $T^l$ is the duration of $l$-th DMP, ~$\delta{ t }$ is the sampling rate, $l=1,\ldots,L$, and $L$ is the number of movement primitives to merge. Note that  \eqref{eq:lineardelayedgoalfunction} generates a piecewise linear moving target $\bfp_m^l$ that reaches the goal $\bfp_d^l$ after $T^l\,$s (see Fig. \ref{fig:clock_and_goal}(a)). Being $\bfp_m^l = \bfp^l(0)$, the acceleration \eqref{eq:dmp_lin_acc_3} is smooth at the beginning of the motion. For this reason, the term $\bfd_0^p(s)$ used in~\eqref{eq:dmp_lin_acc} is not needed in~\eqref{eq:dmp_lin_acc_3}.}
The non-linear forcing term $\bff^p(h)$ {used in~\eqref{eq:dmp_lin_acc_3}} slightly differs from the one in \eqref{eq:force_term} 
\small
\begin{equation}
\bff^{p}(h) = \frac{\sum_{i=1}^{N}\bfw_{i}\psi_{i}(t)}{\sum_{i=1}^{N}\psi_{i}(t)}h,\quad \psi_{i}(t) = e^{-(\textcolor{blue}{\frac{ t }{ \tau T }}-c_{i})^2\textcolor{blue}{/2\sigma^{ 2 }_{ i }}}, \label{eq:force_term_2}
\end{equation}
\normalsize
where~$\sigma_{ i }$ is the width of the~$i$-th kernel,~$c_{ i }$ are their centres, and $h$ is given by \eqref{eq:sigmoidaldecay}. The kernels $\psi_i$ in \eqref{eq:force_term_2} differ from those in \eqref{eq:force_term} since the term $t/\tau T$ replaces the canonical system $h$. Note that, for $\tau=1$, $0\leq t/\tau T \leq 1$ and the kernels are equally spaced between $0$ and $1$.  The kernel widths~$\sigma_{i}$ are constant and depend on the number of kernels. 
 
Given $L$ DMPs in the described form, one can obtain a single DMP by combining kernels and weights of the separately learned DMPs. In particular, the centers, originally equally spaced between 0 and 1, are replaced by 
\small
\begin{equation} 
\overline{c}^{ l }_{ i }=\begin{cases} \frac{ T^{ 1 }( i-1 )}{ T( N-1 )}, &l=1\\
\frac{ T^{ l }( i-1 )}{ T( N-1 ) } + \frac{ 1 }{ T} \sum\limits_{ k=1 }^{ l-1 }{ T^{ k } }, & \text{otherwise}
 \end{cases}, \label{eq:centres} 
\end{equation}
\normalsize
where $N$ is the number of kernels of each DMP, $i = 1, ..., N$, $l = 1, ..., L$, $T^{l}$ is the duration of the $l$-th DMP, and ~$T = \sum_{k=1}^{L}{ T^{k}}$ is the duration of the joint trajectory. 
The width of the kernels is scaled down by $T^l/T$, i.e. $\overline{\sigma}^{l}_{l}= \sigma^{l}_{i}T^{l}/T$.
The weights of each DMP $\bfw_i^l$ remain unchanged. 
{The $N$ kernels and $N$ weights of the $L$ DMPs are stacked together to form a single DMP with $NL$ kernels and $NL$ weights.} The combined kernels of succeeding DMPs now intersect at the transition points, resulting in smooth transitions. 

We extend the described approach to unit quaternions. The angular acceleration in \eqref{eq:dmp_ang_acc} is rewritten as 
\small
$\tau \dot{\bfomega} = \bfK^q\bfe_o(\textcolor{blue}{\bfq_m^l} ,\bfq) + \bfK^q\bff^q(h) - \bfD^q\bfomega.$
\normalsize
 The quaternion goal function $\bfq_m^l$ ranges from $\bfq^l(0)$ to $\bfq_d^l$ in $T^l\,$ seconds {(see Fig. \ref{fig:clock_and_goal}(b)).} 
Hence, $\bfq_m^l$ is a geodesic on $\mathcal{S}^3$ and it is defined as $\small\bfq_m^l(t+1) = \exp\left(\frac{\tau\bfomega_m^l}{2}\right)\ast\bfq_m^l(t)\normalsize$, where
\small
\begin{equation} 
\bfomega_m^l = \begin{cases}
\frac{2}{T^l}\log(\bfq_d^l \ast \bfq^l(0)) & \sum\limits_{ k=1 }^{ l-1 }{ T^{ k } \leq t \leq \sum\limits_{ k=1 }^{ l }{ T^{ k } }}\\  [0,\,0,\,0]\tr, & \text{ otherwise } \end{cases}. \label{eq:qdelayedgoalfunction}
\end{equation}
\normalsize
%
In \eqref{eq:qdelayedgoalfunction}, $\bfq_d^l$ is the goal and $\bfq^l(0)$ is the initial orientation of the $l$-th DMP, ${T^l}$ is the time duration of the $l$-th DMP, $l=1,\ldots,L$, and $2\log(\bfq_d^l \ast \bfq^l(0))$  is the angular velocity that rotates $\bfq^l(0)$ into $\bfq_d^l$ in a unitary time. The functions $\log(\cdot)$ and  $\exp(\cdot)$ are defined in \eqref{eq:logarithmic_map} and \eqref{eq:exponential_map} respectively. Note that the described approach to calculate $\bfq_m^l$ corresponds to interpolate $\bfq^l(0)$ and $\bfq_d^l$ with the SLERP algorithm \cite{slerp}. To reach the final orientation $\bfq^{L}_d$, the delayed goal function firstly reaches $\bfq^{1}_d$, then $\bfq^{2}_d$, and so on until $\bfq^{L}_d$ is reached.

%% file: sections/Experiments.tex
\section{Experimental Results}\label{sec:experiments}
\begin{figure*}[t]
	\centering
	{%
    \subfigure[Quaternion]{\centering {\includegraphics[width=0.16\textwidth]{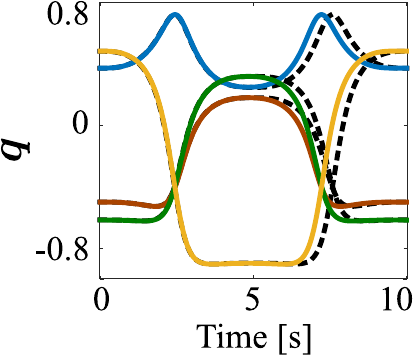}}}
    \subfigure[Angular velocity]{\centering {\includegraphics[width=0.16\textwidth]{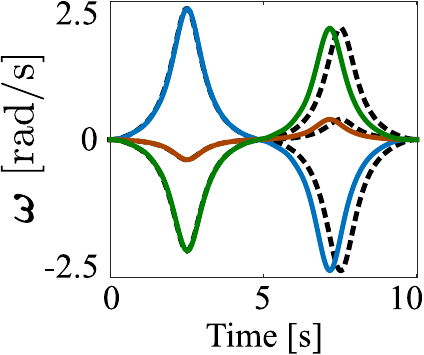}}}
    \subfigure[Angular velocity]{\centering {\includegraphics[width=0.16\textwidth]{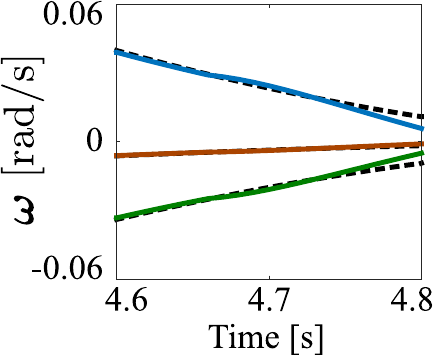}}}
    \subfigure[Goal switch]{\centering {\includegraphics[width=0.16\textwidth]{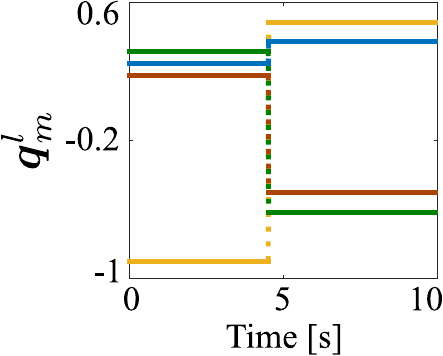}}}
    \subfigure[Orientation error]{\centering {\includegraphics[width=0.17\textwidth]{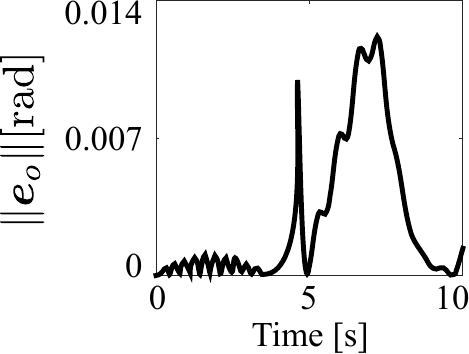}}}
    \vspace{-0.2cm}
    
    \subfigure[Quaternion]{\centering {\includegraphics[width=0.16\textwidth]{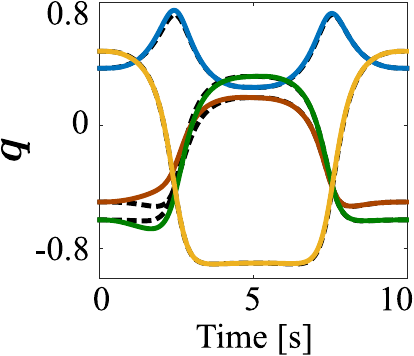}}}
    \subfigure[Angular velocity]{\centering {\includegraphics[width=0.16\textwidth]{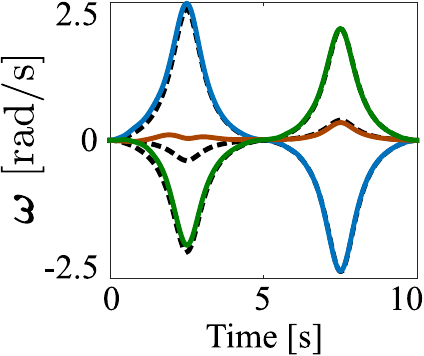}}}
    \subfigure[Angular velocity]{\centering {\includegraphics[width=0.16\textwidth]{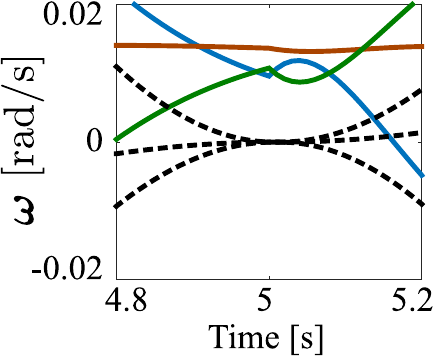}}}
    \subfigure[Moving target]{\centering {\includegraphics[width=0.16\textwidth]{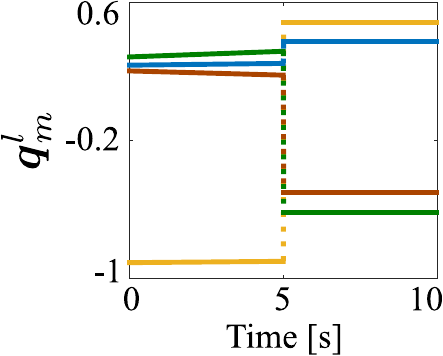}}}
    \subfigure[Orientation error]{\centering {\includegraphics[width=0.16\textwidth]{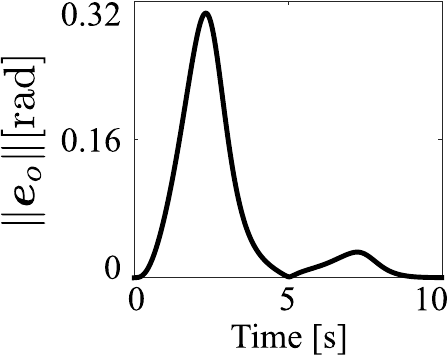}}}
    \vspace{-0.2cm}
    
    \subfigure[Quaternion]{\centering {\includegraphics[width=0.16\textwidth]{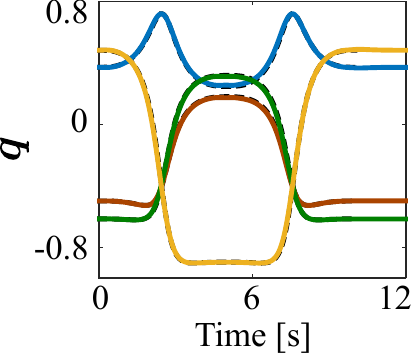}}}
    \subfigure[Angular velocity]{\centering {\includegraphics[width=0.16\textwidth]{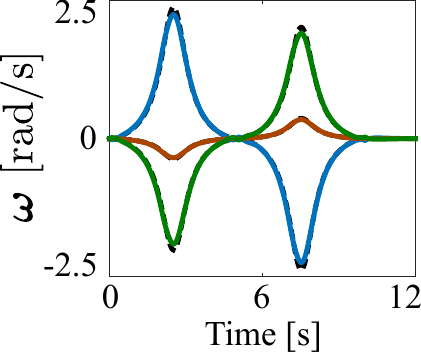}}}
    \subfigure[Angular velocity]{\centering {\includegraphics[width=0.16\textwidth]{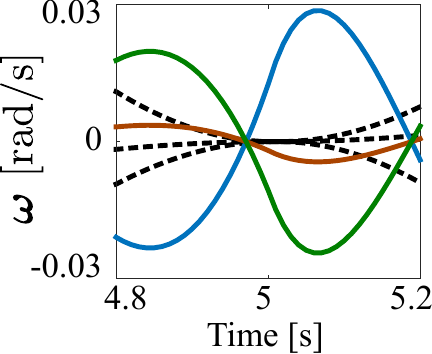}}}
    \subfigure[Delayed goal]{\centering {\includegraphics[width=0.16\textwidth]{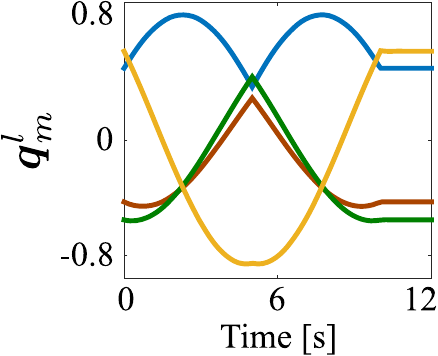}}}
    \subfigure[Orientation error]{\centering {\includegraphics[width=0.16\textwidth]{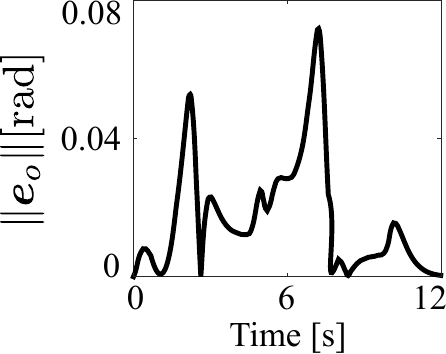}}}
    }%
    \vspace{-0.2cm}
    	\caption{Results obtained by applying the first (a)--(e), second (f)--(j), and third (k)--(o) approach to merge two DMPs trained on synthetic data.}
    	\label{fig:sim_3}
 \vspace{-0.5cm}
\end{figure*} 
\subsection{Synthetic data}
The aim of this experiment is to compare the behaviour of the proposed approaches when applied to generate an orientation trajectory. To this end, we pre-trained two orientation DMPs on synthetic data given by  two minimum jerk trajectories connecting $\bfq(0) = [0.247, 0.178, 0.318, -0.897]\tr$ with $\bfq_d^1 = [0.372, -0.499, -0.616, 0.482]\tr$  (intermediate goal) and $\bfq_d^1$ with $\bfq_d^2 = \bfq(0)$ (final goal). Each trajectory lasts for $T^1 = T^2 = 5\,$s (black dashed lines in Fig. \ref{fig:sim_3}(a)). Each DMP has $N=15$ kernels, $\tau = 1$, and $\bfK^q=10\bfI$. {These values are empirically set, while} $\bfD^q = 2\sqrt{\bfK^p}$ to have a critically damped system \cite{Weitschat16}. The sampling time is $\delta t = 0.01\,$s. The two orientation DMPs are trained to reach the respective goals $\bfq_d^1$ and $\bfq_d^2$ with zero velocity. We apply the approaches presented in Sec. \ref{sec:merge_dmp} to generate a smooth quaternion trajectory that starts and ends at $\bfq(0)$ while passing close to the ``intermediate goal'' $\bfq_d^1$, considered a via point. The performance of each approach is evaluated considering deformation, smoothness, and duration of the generated trajectory, as well as the distance to $\bfq_d^1$. 

Results obtained with the three approaches are shown in Fig. \ref{fig:sim_3}. For the first approach, we switch to the second DMP when the distance from the intermediate goal (via point) $\bfq_d^1$ is below $d_1=0.01\,$rad, i.e. after about $4.7\,$s (Fig. \ref{fig:sim_3}(d)). Alternatively, one can switch the primitives when the velocity is below a certain threshold as suggested in \cite{Pastor_09}. For the second approach, we run the first DMP for $T^1=5\,$s and then switch to the second one. The desired intermediate velocity is $\bfomega_d^1 = [0.01, 0.01, 0.01]\tr$~rad/s. The third approach does not require a switching rule between the DMPs and automatically treats $\bfq_d^1$ as a via point. As expected, all the generated trajectories converge to $\bfq_d^2$ (Fig. \ref{fig:sim_3}(a), (f), and (k)).

Error plots in Fig. \ref{fig:sim_3}(e), (j), and (o) show the deformation introduced by each approach. The first approach is the most accurate (maximum tracking error $e_{o,max}=0.012\,$rad), followed by the third approach ($e_{o,max}=0.072\,$rad). The second approach is the less accurate ($e_{o,max}=0.307\,$rad). The second approach partially sacrifices the accuracy to cross the via point $\bfq_d^1$ after $T^1\,$s ($\bfe_o(T^1)=0.001\,$rad) with velocity $\bfomega(T^1) \approx \bfomega_d^1$ (Fig. \ref{fig:sim_3}(h)). On the other hand, the third approach favors the overall accuracy passing $0.025\,$rad away from $\bfq_d^1$. The trajectory pass ``close'' to the intermediate goal, but the distance depends on the weights of the merged DMPs and cannot be decided a priori. In the first approach, the distance from intermediate goals is a tunable parameter. 

In the first approach, the distance to the goal affects the time duration of the generated trajectory. With the used distance $d_1=0.01\,$rad, the generated trajectory converges to $\bfq_d^2$ (distance below $0.001\,$rad) after $9.5\,$s. Hence, the execution is faster than the demonstration ($T=10\,$s). Bigger values of $d_1$ result in shorter trajectories and vice versa. The second approach produces a trajectory that, as the training data, converges in $10\,$s. Finally, the third approach generates a trajectory of $11.7\,$s, that is $1.7\,$s longer than the demonstrated one. In general, the third approach produces a trajectory that lasts more than the demonstration. This is because the sigmoidal clock signal in \eqref{eq:sigmoidaldecay}---and the effects of the forcing term---vanishes after $T+\delta\,$s, where $\delta$ depends on the steepness $\alpha_h$ of the sigmoid ($\alpha_h=1$ in this case).  Bigger values of $\alpha_h$ result in shorter trajectories and vice versa. 

All the tested approaches are able to generate smooth orientation trajectories with continuous velocities. Nevertheless, the third approach is the only one capable of generating continuous accelerations, while the others may create discontinuous accelerations around the switching point.

\begin{table*}[t!]
     \centering
     \caption{Comparison of the proposed approaches for motion primitives merging.}
     \vspace{-0.2cm}
     \label{tab:features}
     {\renewcommand\arraystretch{1.3} 
 	\begin{tabular}{ c||c|c|c|c|c }
 	 & Intermediate & Desired switch  & Change & {Smooth }  & Computational complexity \\
 	 &  goal crossing & velocity & DMP structure & {motion }  &  wrt original DMP \\
     \hline
     \hline 
     \textbf{First approach} & No & No & No &  Continuous velocity & Same \\
     \textbf{Second approach} & Yes & Yes & Yes & Continuous velocity & Same \\
     \textbf{Third approach} & No & No & Yes & Continuous acceleration & Linear with the number of DMPs\\
 \end{tabular}
 }
 \vspace{-0.5cm}
 \end{table*}
\subsection{Robot experiment}
This experiment compares the merging approaches in a real case where a robot adds sugar into a cup (see Fig. \ref{fig:snapshot}). The task consists of three motion primitives, namely \textit{1)} reach the sugar bowl and fill the spoon, \textit{2)} put the sugar into the cup, and \textit{3)} reach a final pose. The task is demonstrated by kinesthetic teaching and motion primitives are segmented using zero velocity crossing \cite{Fod02} with a velocity threshold {empirically set to} of $5\,$mm/s. The three DMPs are separately learned to reach the relative goal (last pose in the segment) with zero velocity. We use the same parameters as in the previous case. The robot is able to execute the task by stopping at each intermediate goal, but this takes $24.7\,$s that is $5\,$s longer than the demonstration. {Depending on the task at end, the longer execution time may cause issues. Therefore, we also consider the accuracy of the executionmovements and the success of the task to compare the different merging approaches.}  
\begin{figure}[b]
 \vspace{-0.5cm}
	\centering
	\includegraphics[width=\columnwidth]{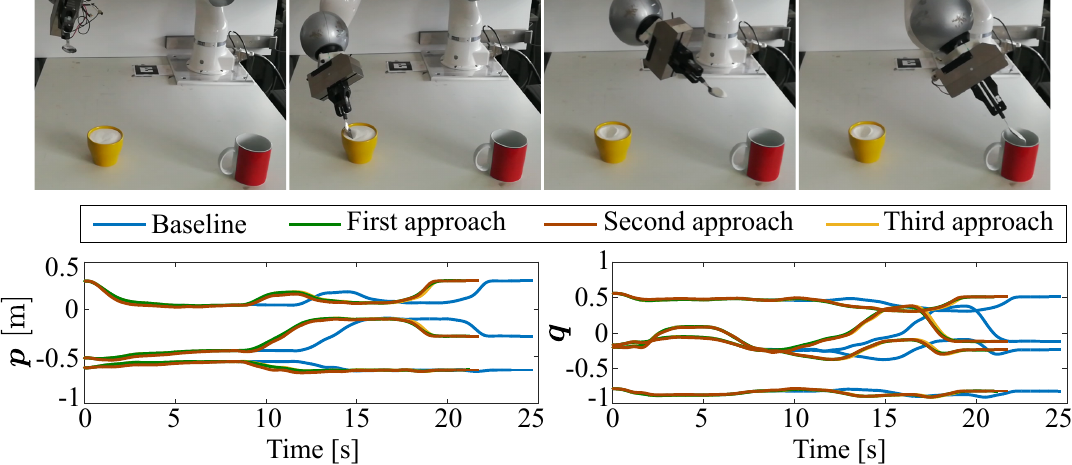}
	\vspace{-0.6cm}
    	\caption{{(Top) Successful execution of the add sugar task. (Bottom) Pose trajectories executed by the robot.}} 
    	\label{fig:snapshot}
\end{figure} 
For the first approach, we switch the DMP when the distance to the current intermediate goal is below $0.005\,$m (rad), allowing the robot to successfully execute the task in $19.4\,$s. The generated trajectory passes close to the intermediate goals (distance below $0.005\,$m (rad)) and accurately represents the demonstration (maximum errors are  $e_{p,max}=0.006\,$m and $e_{o,max}=0.035\,$rad). For the second approach, we set the desired crossing velocity to $0.005\,$m/s (rad/s) along each direction. The robot is able to cross the goals (distance below $0.001\,$m and $0.002\,$rad) but it hits the sugar bowl and fails the task (maximum errors are  $e_{p,max}=0.052\,$m and $e_{o,max}=0.073\,$rad). {The reason is that when the robot reaches the desired $z$ position it is outside the cup and then it touches the cup while reaching the desired $x$-$y$ position (Fig. \ref{fig:snapshot}).} It is worth noticing that the robot is able to execute the task if the crossing velocity is reduced, but this will increase the total execution time. The third approach allows the robot to successfully execute the task in $21.2\,$s. The generated trajectory passes close to the intermediate goals ($0.001\,$m and $0.02\,$rad from the first goal, $0.002\,$m and $0.015\,$rad from the second goal) and accurately represents the demonstration ($e_{p,max}=0.018\,$m and $e_{o,max}=0.05\,$rad).       
\subsection{Discussion}
Presented results on synthetic and real data have shown similarities and differences of the three merging approaches. Important features of each approach are summarized in Tab. \ref{tab:features}. The approach in Sec. \ref{subsec:first_approach} generates a trajectory that converges before the demonstration time. {The faster convergence may represent a problem, for instance when multiple robots are executing a cooperative task. This issue can be alleviated by increasing the time scaling factor $\tau$ to match the demonstrated time.} Among the three approaches, the first one is the easiest to implement since it does not require any change in the DMP structure and in the learning process. On the other hand, approaches two and three requires a moving target and, for approach two, a goal velocity. Hence, the first approach is preferable if standard DMPs were trained {and if the goal of a DMP corresponds to the start of the next one}. 

The second approach introduces a deformation in the generated trajectory. Depending on the desired final velocity, this deformation may not be negligible for the task at hand---as in the presented experiment where the robot touched the sugar bowl and failed the task. However, the second approach is the only capable of crossing the intermediate goals with a user defined velocity. {As shown in Fig. \ref{fig:sim_3}(c), (h), and (m), the second approach is the only one capable of crossing the goal with a user defined velocity.} This is of importance in dynamic tasks like hitting or batting. According to our analysis, the second approach is the best suited for such dynamic tasks.

{The third approach stacks kernels and weights of $L$ trained DMPs into one DMP. Assuming that each DMP has $N$ kernels, the resulting DMP has $NL$ kernels and $NL$ weights. Therefore,} the computational complexity of the third approach grows linearly with the number of DMPs, while the other two approaches have the same cost of a single DMP. {From a certain value of $L$ and $N$, that depends on the available hardware, the third approach is not able to generate the motion in real-time---typically $1$ to $10$~ms}. To alleviate this issue, one can start generating the trajectory using only the kernels of the first two DMPs. The kernels overlaps only in a neighborhood of the intermediate goal. Hence, after passing the intermediate goal, the kernels of first DMP can be replaced with those of the third one, and so on until the last primitive is reached. 
{The third approach is the only one that generates continuous accelerations. Compared to the first approach (Fig. \ref{fig:sim_3}(e)), the third approach slightly deviates from the demonstrated trajectory (Fig. \ref{fig:sim_3}(o)) because training data are smoothen to generate smooth accelerations. The third approach outperforms the first one if the velocity of the successive DMP is different from zero. In this case, the first approach starts the second DMP with a velocity close to zero which causes inaccuracies in reproducing the demonstration. The third approach, instead, generates a velocity at the switching point that is closer to the demonstrated one, resulting in a more accurate trajectory.}

%% file: sections/Conclusions.tex
\section{Conclusion}\label{sec:conclusion}
We presented three approaches to combine a set of motion primitives and generate a smooth trajectory for the robot. The approaches assume that each motion primitive is represented via second-order dynamical systems, the so-called dynamic movement primitives. In contrast to similar work in the field, we consider the orientation part of the motion. We represent the orientation via unit quaternions and exploit the mathematical properties of the quaternion space to rigorously define all the operations required to merge sequential movements. Presented approaches are evaluated both on synthetic and real data, showing that each approach has some distinctive features which make it well suited for specific tasks.  
{In the future, we plan to integrate the motion primitives merging approaches with the symbolic task compression in \cite{saveriano19} allowing a smooth execution of structured tasks.} 

%% file: sections/Appendix.tex
\appendices
\section{}\label{eq:app_1}
Unit quaternions are elements of $\mathcal{S}^3$, the unit sphere in the $3$D space. A unit quaternion has four elements $\mathcal{Q} \triangleq \{\eta,\,\bfepsilon\}\in \mathcal{S}^3$, where $\eta$ is the scalar and $\bfepsilon$ is the vector part of the quaternion. The constraint $\eta^2 + \Vert \bfepsilon\Vert^2 = 1$ relates the scalar and the vector parts. For implementation reasons, a quaternion is represented as a $4$D vector $\bfq \triangleq [\eta, \, \bfepsilon\tr]\tr = [\eta,\,\epsilon_1,\,\epsilon_2,\,\epsilon_3]\tr$. 
The product of two quaternions is
\small
\begin{equation}
\bfq_1 \ast \bfq_2 =[\eta_1\eta_2 - \bfepsilon_1\tr\bfepsilon_2,\,(\eta_1\bfepsilon_2+\eta_2\bfepsilon_1+\bfS(\bfepsilon_1)\bfepsilon_2)\tr]\tr, \label{eq:quaternion_product}
\end{equation}
\normalsize
where $\bfS(\bfepsilon)\in\mathbb{R}^{3\times3}$ is a \textit{skew-symmetric} matrix.
The conjugate of a quaternion, i.e. the quaternion $\overline{\bfq}$ such that $\bfq\ast\overline{\bfq} = [1,\,0,\,0,\,0]\tr$, is defined as $\overline{\bfq} \triangleq [\eta,\,-\bfepsilon\tr]\tr \in \mathcal{S}^3$.
The time derivative of a quaternion is related to the angular velocity by the so-called \textit{quaternion propagation}
\small
\begin{equation}
\dot{\bfq} = \frac{1}{2}\tilde{\bfomega} \ast \bfq \rightarrow
\dot{\bfq} = \begin{bmatrix} \dot{\eta} \\ \dot{\bfepsilon} \end{bmatrix} = \left\lbrace
\begin{split}
\dot{\eta} &= -\frac{1}{2}\bfepsilon\tr\bfomega\\
\dot{\bfepsilon} &= \frac{1}{2}\left(\eta\bfI - \bfS(\bfepsilon)\right)\bfomega
\end{split}\right. ,
\label{eq:quaternion_propagation}
\end{equation}
\normalsize
where $\tilde{\bfomega} = [0, \bfomega\tr]\tr$ is a quaternion with zero scalar part and the angular velocity as vector part. 
The \textit{logarithmic map} 
\small
\begin{equation}
\bfr = \log(\bfq)
\begin{cases}
\arccos(\eta)\frac{\bfepsilon}{\Vert\bfepsilon\Vert},&\Vert\bfepsilon\Vert>0\\
[0,\,0,\,0]\tr,&\text{otherwise}
\end{cases} .
\label{eq:logarithmic_map}
\end{equation}
\normalsize
transforms a unit quaternion into a rotation vector $\bfr\in\mathbb{R}^{3}$. The logarithmic map is uniquely defined and continuously differentiable if the domain is limited to $\mathcal{S}^3/\{-1,\,[0,\,0,\,0]\tr\}$. A rotation vector is mapped into a unit quaternion by the \textit{exponential map} 
\small
\begin{equation}
\exp(\bfr) = 
\begin{cases}
\left[\cos(\Vert\bfr\Vert),\, \sin(\Vert\bfr\Vert)\frac{\bfr\tr}{\Vert\bfr\Vert}\right]\tr, &\Vert\bfr\Vert>0\\
[1,\,0,\,0,\,0]\tr,&\text{otherwise}
\end{cases} .
\label{eq:exponential_map}
\end{equation}
\normalsize
The exponential map is uniquely defined and continuously differentiable if the domain is limited to $0 \leq\Vert\bfr\Vert<\pi$.
The quaternion derivative \eqref{eq:quaternion_propagation} is integrated using the formula
\small
\begin{equation}
\begin{split}
\bfq(t+1) = \exp\left(\frac{\delta t}{2} \bfomega(t) \right) \ast \bfq(t),
\end{split}
\label{eq:quaterion_rate_integration}
\end{equation}
\normalsize
where $\delta t$ is the sampling time.